\documentclass[10pt,twocolumn,letterpaper]{article}

\usepackage{iccv}
\usepackage{times}
\usepackage{epsfig}
\usepackage{graphicx}
\usepackage{amsmath}
\usepackage{amssymb}
\usepackage{booktabs} 

\usepackage{longtable}
\usepackage{amsmath,epsfig,amssymb,graphicx,bm}
\usepackage{subcaption}
\usepackage{epstopdf}
\usepackage{listings}
\usepackage[heightadjust=all]{floatrow}
\usepackage{mathtools}
\usepackage[symbol]{footmisc}
\usepackage{amsthm}
\usepackage[heightadjust=all]{floatrow}
\newfloatcommand{capbtabbox}{table}[][\FBwidth]
\usepackage{subcaption}

\theoremstyle{plain}

\theoremstyle{definition}
\newtheorem{definition}{Definition}[section]

\theoremstyle{definition}
\newtheorem{example}{Example}[section]

\theoremstyle{plain}
\newtheorem{lemma}{Lemma}[section]

\theoremstyle{plain}
\newtheorem{corollary}{Corollary}[section]

\theoremstyle{definition}
\newtheorem{proposition}{Proposition}[section]

\usepackage{tikz, pgfplots}
\pgfplotsset{compat=1.14}
\usepgfplotslibrary{groupplots}
\usepackage[export]{adjustbox}
\usetikzlibrary{positioning, arrows.meta, backgrounds}

\definecolor{main1}{RGB}{230, 30, 100} 
\definecolor{main2}{RGB}{0, 140, 210} 
\definecolor{main3}{RGB}{240, 170, 0} 
\definecolor{main4}{RGB}{67, 176, 71} 
\definecolor{main5}{RGB}{136, 86, 167} 

\newcommand{\addplotErrorBar}[4]{
    \addplot [
        color = #4,
        line width = 4pt,
        mark = none,
        error bars/y dir=both,
        error bars/y explicit,
        error bars/error bar style={line width = 2pt}
        ] table [
        x = #1,
        y = #2median,
        y error minus = #2lower,
        y error plus = #2upper,
        col sep = comma]
        {runtime-figures/fig-analytical-runtime/#1j#3.csv};
}

\newcommand{\addplotLoss}[4]{
    \addplot [
        color = #4,
        ultra thick,
        mark = none,
        ] table [
        x = #2,
        y = #3,
        col sep = comma]
        {loss-figures/#1_loss.csv};
}

\definecolor{mncolor}{RGB}{255,50,00}

\newcommand\blfootnote[1]{%
  \begingroup
  \renewcommand\thefootnote{}\footnote{#1}%
  \addtocounter{footnote}{-1}%
  \endgroup
}

\PassOptionsToPackage{hyphens}{url}\usepackage[pagebackref=true,breaklinks=true,letterpaper=true,colorlinks,bookmarks=false]{hyperref}

\iccvfinalcopy 

\def\iccvPaperID{3695} 

\ificcvfinal\fi
\begin{document}

\title{DDSL: Deep Differentiable Simplex Layer for Learning Geometric Signals}

\renewcommand*{\thefootnote}{\fnsymbol{footnote}}
\author{Chiyu ``Max" Jiang\footnote[1]{\label{ec}Equal contributions}~~\textsuperscript{1}\\
\and
Dana Lansigan\footnotemark[1]~~\textsuperscript{1}\\
\and
Philip Marcus\textsuperscript{1}\\
\and
Matthias Nie{\ss}ner\textsuperscript{2}\\
\and
\textsuperscript{1}UC Berkeley \\
\and
\textsuperscript{2}Technical University of Munich
}
\renewcommand*{\thefootnote}{\arabic{footnote}}
\maketitle
\ificcvfinal\thispagestyle{empty}\fi

\begin{abstract}
We present a Deep Differentiable Simplex Layer (DDSL) for neural networks for geometric deep learning. The DDSL is a differentiable layer compatible with deep neural networks for bridging simplex mesh-based geometry representations (point clouds, line mesh, triangular mesh, tetrahedral mesh) with raster images (e.g., 2D/3D grids). The DDSL uses Non-Uniform Fourier Transform (NUFT) to perform differentiable, efficient, anti-aliased rasterization of simplex-based signals. We present a complete theoretical framework for the process as well as an efficient backpropagation algorithm. Compared to previous differentiable renderers and rasterizers, the DDSL generalizes to arbitrary simplex degrees and dimensions. In particular, we explore its applications to 2D shapes and illustrate two applications of this method: (1) mesh editing and optimization guided by neural network outputs, and (2) using DDSL for a differentiable rasterization loss to facilitate end-to-end training of polygon generators. We are able to validate the effectiveness of gradient-based shape optimization with the example of airfoil optimization, and using the differentiable rasterization loss to facilitate end-to-end training, we surpass state of the art for polygonal image segmentation given ground-truth bounding boxes.
\end{abstract}

\section{Introduction}
\label{sec:intro}
\blfootnote{\textsuperscript{*} Equal contributions}
The simplicial complex (i.e., simplex mesh) is a flexible and general representation for non-uniform geometric signals. Various commonly-used geometric representations, including point clouds, wire-frames, polygons, triangular mesh, tetrahedral mesh etc., are examples of simplicial complexes. Leveraging deep learning architectures for such non-uniform geometric signals has been of increasing interest, and varied methodologies and architectures have been presented to deal with varied representations \cite{bronstein2017geometric}. 

\begin{figure}
    \centering
    \input{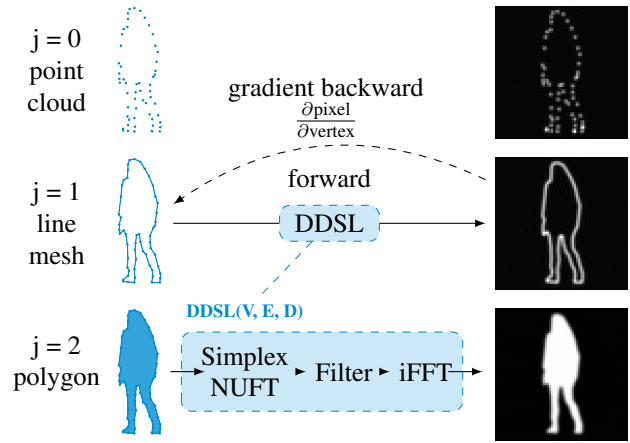}
    \vspace{-0.7cm}
    \caption{A schematic of the DDSL layer with 2D simplex meshes. The DDSL algorithm is general for handling simplex meshes of arbitrary dimensions and simplex degrees. The input to DDSL is a simplex mesh described by three matrices: float matrix V for vertex coordinates, uint matrix E for simplex connectivity, and float matrix D for per-simplex density (constant density of 1 in the example above). A raster image of arbitrary resolution can be produced. The gradient of per-pixel intensity with respect to each spatial coordinate in V can be computed analytically within the DDSL layer.}
    \label{fig:ddsl_schematic}
\end{figure}

In this study, we propose a Deep Differentiable Simplex Layer (DDSL), which performs differentiable rasterization of arbitrary simplex mesh-based geometric signals. The DDSL is based upon simplex Non-Uniform Fourier Transform (NUFT) \cite{jiang2018convolutional} for the forward-pass, which is highly generalizable across arbitrary topologies. Furthermore, we find the general differential form of the simplex NUFT, allowing for an efficient backward pass. Our work differs from previous work in the literature on differentiable rendering in two major ways. First, our network is generalizable across arbitrary simplex degrees and dimensions, making it a unified framework for a range of geometric representations. Second, while other differentiable renderers are specifically posed for projective-rendering by projecting 3D meshes to 2D grids, the DDSL is capable of in-situ rasterization in the original dimension. Building on the differentiable nature of the rasterizer, we explore two unique use cases. First, using the differentiablity of the DDSL, we can utilize Convolutional Neural Network (CNN) based deep learning models as surrogate models of physical properties for shape optimization, which is useful in a range of engineering disciplines. Secondly, using the DDSL as a neural network layer, we can formulate a differentiable rasterization loss that allows for end-to-end generation of shapes using a direct supervised approach, which can be useful in a range of computer vision problems.

As an example of the two use cases, we perform three experiments. First, to validate the effectiveness of gradient propagation through the layer, we illustrate with the toy problem of MNIST shape optimization, where we can use gradients propagated through the neural network and DDSL to manipulate and transform the input polygon mesh into a target digit (Sec. \ref{ssec:optim}). Next, to further illustrate potential applications of neural shape optimization enabled by the DDSL, we investigate the classic engineering problem of airfoil optimization and show that the shape optimization pipeline effectively manipulates the input shape into a desired lift-drag ratio (Sec. \ref{ssec:optim}). Finally, to illustrate the effectiveness of the differentiable rasterization loss, we train a polygon generating neural network end-to-end with direct supervision to generate polygonal segmentation masks for image segmentation (Sec. \ref{ssec:seg}). With the novel rasterization loss, we surpass state-of-the-art in the polygon segmentation task, with a much simpler network architecture and training scheme.

In summary, we contribute the following:
\begin{itemize}
    \itemsep0em
    \item We propose the DDSL, which is a differentiable rasterizer for arbitrary simplex-mesh based geometries. Its differentiable nature allows for its effective integration in deep neural networks.
    \item We show that the DDSL effectively facilitates shape optimization for engineering applications such as aerodynamic optimzation of airfoils, using neural networks as surrogate models.
    \item We show that the DDSL can be used to produce a differentiable rasterization loss, which can be used to create direct supervision to facilitate end-to-end training of  shape generators, with applications in polygonal segmentation mask generation.
    \item We develop and release code for effectively integrating the DDSL into deep neural networks\footnote{Code available: \scalebox{1.0}[1.0]{\footnotesize \url{https://github.com/maxjiang93/DDSL}}}, with compelling computational performance benchmarks.
\end{itemize}

\section{Related Work}\label{sec:bg}
We present a brief overview of geometric representations for deep learning, various related differentiable renderers, and related work in the space of our two exemplary applications.

\paragraph{Geometric Representations for Deep Learning} In general, there are two classes of geometric representations, either in its native form of simplex meshes, or in a raster form which can be efficiently processed with grid-based network architectures such as CNNs. As simplex meshes come in various forms and dimensions (point clouds, meshes etc.), there is a vast body of literature for different geometric signals of different simplex degrees and dimensions. For example, PointNets have been specially designed for point clouds \cite{qi2017pointnet, qi2017pointnet++}, various algorithms perform convolutions natively on the mesh manifold, \cite{jiang2018spherical, huang2018texturenet, boscaini2016learning}, the graph \cite{defferrard2016convolutional, kipf2016semi, yi2017syncspeccnn} etc. 

Grid-based algorithms on the other hand require the rasterization of a simplex-mesh based geometric signal for further processing by CNNs. Examples of such include binary-voxel based algorithms \cite{maturana2015voxnet, wu2016learning}, Truncated Signed-Distance Function (TSDF) based algorithms \cite{dai2017complete, zeng20173dmatch, song2017semantic, dai2018scancomplete}, multi-view image based algorithms \cite{su2015multi, kanezaki2018rotationnet}, and hybrids \cite{kalogerakis20173d, dai20183dmv}. Compared to deep learning methods that directly perform convolutions on the simplex mesh, grid-based methods are more generalizable across shape topologies and computationally easier to implement, since it leverages highly efficient tensor operators such as 2D/3D convolution kernels for rasterized data. However, conventional voxelization methods are not differentiable with respect to the input mesh, and differentiable rasterizers have been proposed to close the gap between simplex and grid representations.

\paragraph{Differentiable Rasterization in Deep Learning} Recently, a series differentiable projective renderers have been proposed. \cite{loper2014opendr} proposed an approximate differentiable rasterizer for inverse graphics. \cite{kato2018neural} proposed a deep neural renderer that uses linear approximations for the gradients of the pixel intensity with respect to the vertex positions. \cite{li2018differentiable} introduced a differentiable ray-tracer for differentiability of additional rendering effects. Very recently, \cite{liu2019soft} proposed a differentiable rasterizer that approximates rendering derivatives with soft boundaries. Various studies in face mesh reconstruction applications \cite{genova2018unsupervised, tewari2018self, tewari2017mofa, richardson2017learning} and general mesh reconstruction tasks \cite{kanazawa2018learning, kundu20183d} utilize some form of differentiable rasterization to facilitate gradient flows in neural networks. 

\paragraph{Shape Optimization}
Shape optimization is essential in a broad range of engineering fields, including aerodynamic, mechanical, structural, and architectural designs. Traditionally, shape optimization algorithms couple gradient-based or gradient-free optimizers (e.g., genetic algorithms, simulated annealing) with physics simulators, e.g., Computational Fluid Dynamics (CFD) and multiphysics software for evaluation. For aerodynamic shape optimization, the adjoint method has been used for gradient-based optimizations with sensitivities acquired from physics simulators \cite{pironneau1974optimum, jameson1998optimum}. Recently, machine learning algorithms such as multilayer perceptrons have been used as surrogate models for the response surface to speed up evaluation and optimization \cite{khurana2009airfoil, lundberg2015automated}. More recently, CNNs have been used for the evaluation of aerodynamic properties \cite{zhang2018application}, and gradient-based optimization methods coupled with CNNs have been explored \cite{hennigh2017automated}. However, direct manipulation of input mesh has not been achieved due to the lack of in-situ differentiable rasterization of polygons and 3D meshes.

\paragraph{Image Segmentation with Polygon Masks}
Image segmentation is a central task in computer vision, and has been thoroughly studied. Much of the work in the image segmentation literature creates pixel-level masks \cite{long2015fully, ronneberger2015u, wang2018understanding, he2017mask, dai2016instance, li2017fully}. However, more recently, to address the need of assisting human annotators to create ground-truth segmentation labels, new network architectures such as PolygonRNN \cite{castrejon2017annotating} and PolygonRNN++ \cite{acuna2018efficient} have been proposed for creating polygonal segmentation masks given ground-truth bounding boxes. Our work targets this application to explore a more effective and efficient polygon generating network using our DDSL-enabled rasterization loss.

\section{Method}\label{sec:method}

\begin{table}[]
    \centering
    \begin{tabular}{c p{14em}}
        \toprule
        Notation &  Description \\
        \midrule
        $d$ & Dimension of Euclidean space $\mathbb{R}^d$ \\
        $j$ & Degree of simplex. Point $j=0$, Line $j=1$, Tri. $j=2$, Tet. $j=3$ \\
        $n, N$ & Index of the $n$-th element among a total of $N$ elements \\
        $\Omega_n^j$ & Domain of $n$-th element of order $j$ \\
        $\bm{x}$ & Cartesian space coordinate vector. $\bm{x} = (x, y, z)$ \\
        $\bm{k}$ & Spectral domain coordinate vector. $\bm{k} = (u, v, w)$ \\
        $p$ & Index of a point in a simplex element. $p\in\mathbb{N}$, $p\leq j+1$\\
        $i$ & Imaginary number unit\\
        \bottomrule
    \end{tabular}
    \caption{List of math symbols in our method.}
    \label{tab:notation}
\end{table}

\subsection{DDSL Overview}\label{ssec:ddsl}
A schematic of the DDSL layer is presented in Fig. \ref{fig:ddsl_schematic}. The DDSL layer consists of three consecutive mathematical operations, first computing the Fourier transform of the simplicial complex by uniformly sampling it in the spectral domain, followed by a spectral filtering step by multiplying the spectral signal with a Gaussian filter to eliminate ringing effects. Lastly, we use the inverse Fourier Transform (iFFT) to acquire the physical raster image corresponding to the input. Since the forward and backward methods of the filtering step (an element-wise product) and iFFT are well known, we focus our analysis on the simplex NUFT, which we derive and detail below.

\subsection{Mathematical Description}\label{ssec:theory}
We represent discrete geometric signals as weighted simplicial complexes. We provide the following definitions for a $j$-simplex and a $j$-simplex mesh: 

\begin{definition}[$j$-simplex]
A \textit{simplex} is the generalization of the two-dimensional triangle in other dimensions. The $j$-simplex determined by $j+1$ affinely independent points $v_0,\dots,v_j\in\mathbb{R}^n$ is
\begin{align}
    C&=\textbf{conv}\{v_0,\dots,v_j\}\nonumber\\
    &=\{\theta_0 v_0+\dots+\theta_j v_j\ |\ \bm{\theta} \succeq 0,\ \bm{1}^T\bm{\theta}=1\}
\end{align}
where $\bm{1}$ is the vector with all entries one.
\end{definition}

\begin{definition}[$j$-simplex mesh]
A simplicial complex consisting only of $j$-simplices is a homogeneous simplicial $j$-complex, or a \textit{$j$-simplex mesh}.
\end{definition}

\begin{example}[Examples of simplices and simplex meshes]
A $0$-simplex is a point, a $1$-simplex is a line, a $2$-simplex is a triangle, and a $3$-simplex is a tetrahedron. The $0$-, $1$-, $2$-, and $3$-simplicial complexes are the point cloud and linear, triangular, and tetrahedral meshes, respectively.
\end{example}

\begin{definition}[Functions over a $j$-simplex element and a $j$-simplex mesh]
The Piecewise-Constant Function (PCF) over a $j$-simplex mesh consisting of $N$ simplices is the superposition of the density functions $f_n^j(\bm{x})$ for each $j$-simplex with domain $\Omega_n^j$ and signal density $\rho_n$:
\begin{align}
    f_n^j(\bm{x})=
    \begin{cases}
    \rho_n, \bm{x}\in\Omega_n^j \\
    0, \bm{x}\notin\Omega_n^j
    \end{cases},
    \quad
    f^j(\bm{x})=\sum_{n=1}^N f_n^j(\bm{x})
\end{align}
\end{definition}

For the forward pass, we use the NUFT of a PCF over a $j$-simplex mesh.

\begin{proposition}[Forward pass]
The NUFT of a PCF over a simplex in a mesh is
\begin{equation}
    F_n^j(\bm{k})=\rho_ni^j\gamma_n^jS
    \label{eq:Fnj}
\end{equation}
\begin{equation}
    S:=\sum_{t=1}^{j+1}\frac{e^{-i\sigma_t}}{\prod_{l=1,l\neq{t}}^{j+1}(\sigma_t-\sigma_i)}, \quad \sigma_t:=\bm{k} \cdot \bm{x}_t
\end{equation}
where $\gamma_n^j$ is the content distortion factor, which is the ratio between the simplex content and the unit orthogonal simplex content. The simplex content $C_n^j$ is computed using the Cayley-Menger determinant:
\begin{align}
    C_n^j&=\sqrt{\frac{(-1)^{j+1}}{2^j(j!)^2}det(\hat{B}_n^j)}\\
    \hat{B}_n^j&:=\begin{bmatrix}0 & 1 & 1 & 1 & \dots\\
    1 & 0 & d_{12}^2 & d_{13}^2 & \dots \\
    1 & d_{21}^2 & 0 & d_{23}^2 &  \dots \\
    1 & d_{31}^2 & d_{32}^2 & 0 &  \dots \\
    \vdots & \vdots & \vdots & \vdots & 
    \end{bmatrix}
\end{align}
where each element $d_{st}^2$ of $\hat{B}_n^j$ is the squared distance between points $s$ and $t$. The content of the unit orthogonal simplex $C_I^j$ is $1/j!$, so the content distortion factor is
\begin{equation}
    \gamma_n^j=\frac{C_n^j}{C_I^j}=j!C_n^j
\end{equation}
 
\begin{figure*}\BottomFloatBoxes
\begin{floatrow}
\ffigbox[.7\textwidth]{%
    \input{shape-optimization-figure/optim_schematic.tex}
}{%
  \caption{Schematic of deep learning model driven shape optimization pipeline.}
  \label{fig:dlarch1}
}
\ffigbox[.26\textwidth]{%
\begin{tikzpicture}[node distance=1em]
\tikzset{>=latex}
\tikzstyle{wcirc} = [circle, draw=black, fill=white, inner sep=0pt, minimum size=.5em]
\tikzstyle{bcirc} = [circle, draw=black, fill=black, inner sep=0pt, minimum size=.5em]
\tikzstyle{arrow} = [-{Latex[width=3, length=4]}, line width=0.3]
\node [wcirc] (w0) at (0,0) {}; 
\node [wcirc] (w1) at ($ (w0.center) + (4.5em, 7.5em) $) {};
\node [wcirc] (w2) at ($ (w0.center) + (10.5em, 2em) $) {};
\draw [dashed] (w0) -- (w1) -- (w2) -- (w0);

\node [] (e0) at ($ (w0)!0.50!(w1) $) {};
\node [] (e1) at ($ (w1)!0.50!(w2) $) {};
\node [] (e2) at ($ (w2)!0.50!(w0) $) {};

\node [bcirc] (b0) at ($(e0.center) + (0.8574em, -0.5144em)$) {};
\node [bcirc] (b1) at ($(e1.center) + (1em, 1em)$) {};
\node [bcirc] (b2) at ($(e2.center) + (0.2828em, -1.9798em)$) {};

\draw (w0) -- (b0);
\draw (w0) -- (b2);
\draw (w1) -- (b0);
\draw (w1) -- (b1);
\draw (w2) -- (b1);
\draw (w2) -- (b2);

\draw [arrow] (e0.center) -- (b0);
\draw [arrow] (e1.center) -- (b1);
\draw [arrow] (e2.center) -- (b2);

\node [left=0 of e0] () {$\delta_1^{(0)}$};
\node [left=0 of e1] () {$\delta_2^{(0)}$};
\node [above=0 of e2] () {$\delta_3^{(0)}$};
\end{tikzpicture}
}{%
  \caption{Schematic for the hierarchical polygon generation process in PolygonNet. New nodes in the next hierarchy are generated by offsetting edge center in normal direction by $\delta$.}
  \label{fig:polygen}
}
\end{floatrow}
\end{figure*}

\begin{figure*}[ht!]
    \centering
    \input{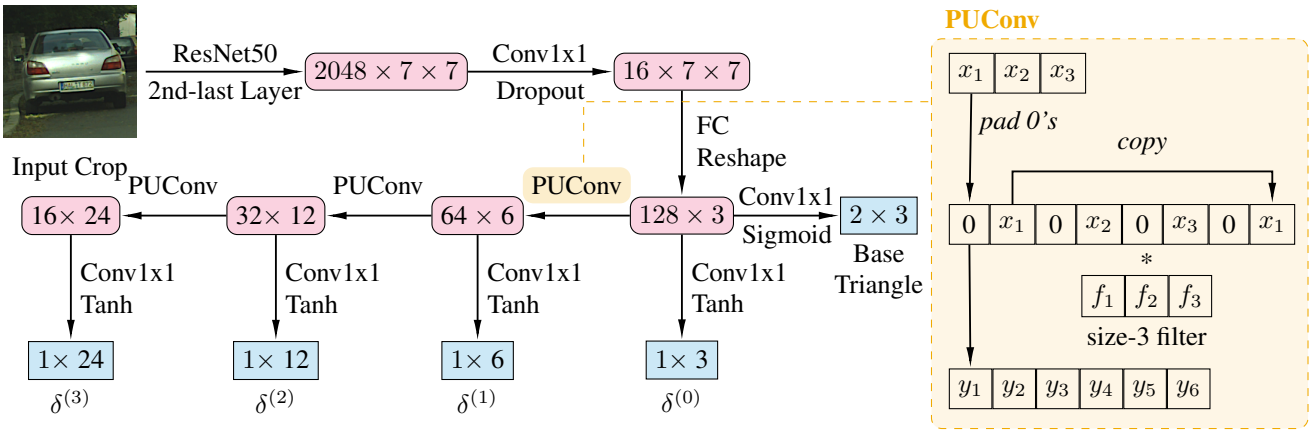}
    \caption{Schematic of the deep learning architecture for polygon segmentation  (PolygonNet). All intermediate layers are followed by BatchNorm and ReLU. A Periodic Upsampling Convolution (PUConv) is used to generate vertex offsets ($\delta$) at the consecutive level. For each level, we learn a learnable scale factor for all offsets.}
    \label{fig:dlarch2}
\end{figure*}

From the linearity of the Fourier transform, the NUFT of a PCF over an entire $j$-simplex mesh is
\begin{equation}
    F^j(\bm{k}) = \sum_{n=1}^N F_n^j(\bm{k}) = \sum_n^N \rho_n i^j\gamma_n^jS
\end{equation}

\end{proposition}

For efficient computing, we use the auxiliary node method (AuxNode), which utilizes signed content.

\begin{corollary}[AuxNode]
To compute the Fourier transform of uniform signals in $j$-polytopes represented by its watertight $(j-1)$-simplex mesh using AuxNode, Eqn. (\ref{eq:Fnj}) is modified as follows: 
\begin{align}
    F_n^j(\bm{k}) =& i^j \sum_{n'=1}^{N'_n}s_{n'}\gamma_{n'}^j\left(\frac{(-1)^{j}}{\prod_{l=1}^{j}\sigma_l}\right.\nonumber\\
    &\left.+\sum_{t=1}^{j}\frac{e^{-i\sigma_t}}{\sigma_t\prod_{l=1, l\neq t}^{j}(\sigma_t-\sigma_l)}\right)
    \label{eqn:finalan}
\end{align}
where $s_{n'}\gamma_{n'}^j$ is the signed content distortion factor for the $n'$th auxiliary $j$-simplex where $s_{n'}\in\{-1,1\}$. For practical purposes, assume that the auxiliary $j$-simplex is in $\mathbb{R}^d$ where $d=j$. The signed content distortion factor is computed using the determinant of the Jacobian matrix for parameterizing the auxiliary simplex to a unit orthogonal simplex:
\begin{align}
    s_{n'}\gamma_{n'}^j = j!\det(J) = j! \det([\bm{x}_1, \bm{x}_2, \cdots, \bm{x}_j])
    \label{eqn:cdfan}
\end{align}
\end{corollary}

\begin{proof}
Refer to \cite{jiang2018convolutional}.
\end{proof}

For the backward pass, we derive the analytic derivative of the NUFT with respect to the vertex coordinates of a j-simplex mesh. Following from the product rule, we require the derivatives of the content distortion factor $\gamma_n^j$ and the summation term $S$ to obtain the entire derivative of $F_n^j(\bm{k})$.

\begin{lemma}[Derivative of the content distortion factor]
\label{lem:dgammadxp}
The derivative of $\gamma_n^j$ with respect to vertex coordinate $\bm{x}_p$ is
\begin{equation}
    \frac{\partial \gamma_n^j}{\partial\bm{x}_p}=\frac{(-1)^{j+1}/2^j}{\gamma_n^j}\sum_{\substack{m=1\\m\neq{p}}}^{j+1}A_{pm}\bm{D}_{pm}
    \label{eq:dgammadxp}
\end{equation}
where $\bm{D}_{pm}=2(\bm{x}_p-\bm{x}_m)$ and $A_{pm}$ is the element in the $(p+1)$th row and $(m+1)$th column of $adj(\hat{B_n^j})$.
\end{lemma}

\begin{figure*}[ht!]
    \centering
    \begin{subfigure}[h]{\textwidth}
    \resizebox{\textwidth}{!}{\begin{tikzpicture}[font = \huge]
    \begin{groupplot}[
    width = 0.6\textwidth,
    group style = {
        group size = 4 by 1,
        xlabels at = edge bottom,
        ylabels at = edge left,
        xticklabels at = edge bottom,
        yticklabels at = edge left},
    xlabel = Number of Points,
    ylabel = Runtime (ms),
    ymin = 1,
    ymax = 2000,
    ymode = log,
    log basis y = {10},
    xmin = 1,
    xmax = 54,
    xtick = {10,20,30,40,50},
    xticklabels = {10,20,30,40,50},
    every outer x axis line/.append style = {black!30, ultra thick},
    every outer y axis line/.append style = {black!30, ultra thick},
    every tick label/.append style = {black!70}]
    \nextgroupplot[title = {$j=0$}]
    \addplotErrorBar{npoints}{atime}{0}{main1}
    \addplotErrorBar{npoints}{ftime}{0}{main2}
    \nextgroupplot[title = {$j=1$}]
    \addplotErrorBar{npoints}{atime}{1}{main1}
    \addplotErrorBar{npoints}{ftime}{1}{main2}
    \nextgroupplot[title = {$j=2$}]
    \addplotErrorBar{npoints}{atime}{2}{main1}
    \addplotErrorBar{npoints}{ftime}{2}{main2}
    \nextgroupplot[title = {$j=3$}]
    \addplotErrorBar{npoints}{atime}{3}{main1}
    \addplotErrorBar{npoints}{ftime}{3}{main2}
    \end{groupplot}
\end{tikzpicture}}
    \caption{}
    \end{subfigure}
    \begin{subfigure}[h]{\textwidth}
    \resizebox{\textwidth}{!}{\begin{tikzpicture}[font = \huge] 
    \begin{groupplot}[
    width = 0.6\textwidth,
    group style = {
        group size = 4 by 1,
        xlabels at = edge bottom,
        ylabels at = edge left,
        xticklabels at = edge bottom,
        yticklabels at = edge left},
    xlabel = Resolution,
    ylabel = Runtime (ms),
    ymin = 1,
    ymax = 2000,
    ymode = log,
    log basis y = {10},
    xmin = 2,
    xmax = 34,
    xtick = {4,8,16,32},
    xticklabels = {4,8,16,32},
    every outer x axis line/.append style = {black!30, ultra thick},
    every outer y axis line/.append style = {black!30, ultra thick},
    every tick label/.append style = {black!70}]
    \nextgroupplot[title = {$j=0$}]
    \addplotErrorBar{res}{atime}{0}{main1}
    \addplotErrorBar{res}{ftime}{0}{main2}
    \nextgroupplot[title = {$j=1$}]
    \addplotErrorBar{res}{atime}{1}{main1}
    \addplotErrorBar{res}{ftime}{1}{main2}
    \nextgroupplot[title = {$j=2$}]
    \addplotErrorBar{res}{atime}{2}{main1}
    \addplotErrorBar{res}{ftime}{2}{main2}
    \nextgroupplot[title = {$j=3$}]
    \addplotErrorBar{res}{atime}{3}{main1}
    \addplotErrorBar{res}{ftime}{3}{main2}
    \end{groupplot}
\end{tikzpicture}}
    \caption{}
    \end{subfigure}
    \caption{Comparison of the analytic (pink) and numeric (blue) derivative runtimes for the (a) mesh size and (b) resolution tests. All rasters are computed for a square cube, and resolution is per dimension.}
    \label{fig:performance}
\end{figure*}
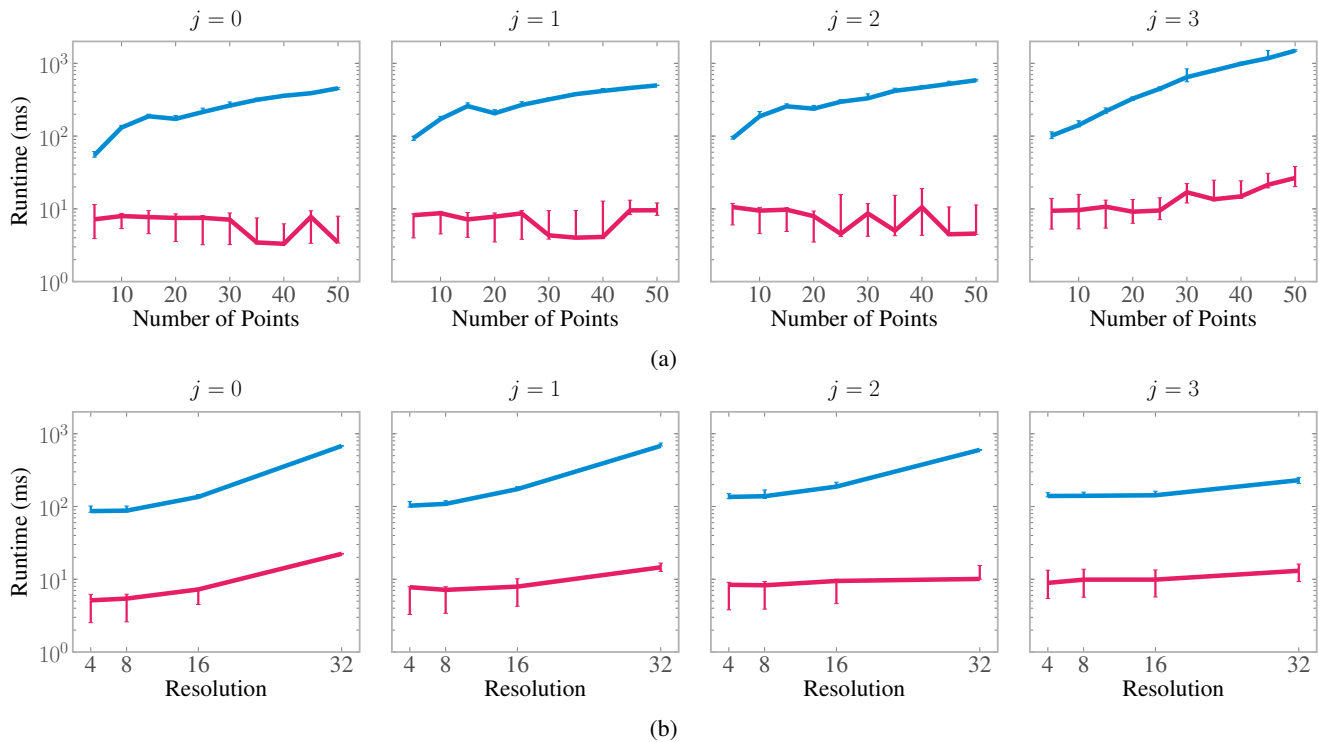

\begin{lemma}[Derivative of the summation term]
\label{lem:dSdxp}
Let $S_t$ be one term in the summation term $S$:
\begin{equation}
    S_t:=\frac{e^{-i\sigma_t}}{\prod_{l=1,l\neq{t}}^{j+1}(\sigma_t-\sigma_l)}
    \label{eq:dSdxp}
\end{equation}
The derivative of the summation term with respect to $\bm{x}_p$ is
\begin{equation}
    \frac{\partial S}{\partial\bm{x}_p}=\left(-i S_p+\sum_{t=1,t\neq p}^{j+1}\frac{S_t+S_p}{\sigma_t - \sigma_p}\right)\bm{k}
    \label{eq:dSdxp}
\end{equation}
where $\bm{k}$ is the spectral domain coordinate vector.
\end{lemma}

\begin{proposition}[Backward pass]
Following from Lemmas \ref{lem:dgammadxp} and \ref{lem:dSdxp}, the derivative of $F_n^j(\bm{k})$ with respect to a point $\bm{x}_p$ in the simplex element $n$ is
\begin{equation}
    \frac{\partial F_n^j(\bm{k})}{\partial\bm{x}_p}=\rho_ni^j\left(\Lambda\bm{k}+\Gamma\sum_{\substack{m=1\\m\neq{p}}}^{j+1}A_{pm}\bm{D}_{pm}\right)
    \label{eq:dFdxp}
\end{equation}
where $A_{pm}$ is the element in the $p$th row and $m$th column of $adj(\hat{B_n^j})$ starting at $p=0$ and $m=0$,
\begin{align}
    \Lambda:=&\gamma_n^j \left(-i S_p+\sum_{t=1,t\neq p}^{j+1}\frac{S_t+S_p}{\sigma_t - \sigma_p}\right)\\
    \Gamma :=& \frac{(-1)^{j+1}/2^j}{\gamma_n^j}S
\end{align}

\end{proposition}

We provide a detailed derivation of Eqn. \ref{eq:dFdxp} as well as proofs of Lemmas \ref{lem:dgammadxp} and \ref{lem:dSdxp} in Sec. \ref{assec:derivative} of the Appendix.

\subsection{Deep Learning Architectures and Pipelines}\label{ssec:dlarch}
We present the a schematic of the deep learning model-driven shape optimization (Sec. \ref{ssec:optim}) in Fig. \ref{fig:dlarch1}, and a schematic of the polygon segmentation network (PolygonNet) in Figs. \ref{fig:polygen} and \ref{fig:dlarch2}. A detailed description of the architectures is presented in Appendix \ref{asec:arch}.

\section{Experiments}\label{sec:exp}
\subsection{Performance Benchmarking}\label{ssec:perf}
We compare the runtime of our implementation of the backward pass over the DDSL with that of the numeric derivatives calculated using the finite difference method. 

\paragraph{Experiment Setup} We perform tests for the 0-, 1-, 2-, and 3-simplex meshes in 3-dimensional space and examine the effects of mesh size (number of points in the mesh) and image resolution.
We test mesh sizes ranging from 5 to 50 points and resolutions ranging from 4 to 32, and we run each test 100 times to acquire a distribution of data.
For each run, we randomly generate a 3-dimensional simplex mesh of varied simplex degrees, varied densities, with random gradient values on each raster pixel.
We then calculate the analytic and numeric derivatives for the DDSL using our implementation of Eqn. \ref{eq:dFdxp} and the finite difference method, respectively, and time each calculation.

\paragraph{Analysis of complexity} Since the analytic finite difference backward pass for computing the gradients using Eqn. \ref{eq:dFdxp} requires computing each pair of spectral coefficient and each vertex in a $j$-simplex, the computational complexity for the finite difference backward pass is the same as the forward pass, $\mathcal{O}((j+1)n_e m)$, for a mesh of $n_e$ simplices and a raster of $m$ degrees of freedom. Finite difference, on the other hand, requires $n_v$ forward computations, each of complexity $\mathcal{O}((j+1)n_e m)$. Assuming $n_v \propto n_e$, the Finite Difference evaluation is of complexity $\mathcal{O}((j+1)n_e^2 m)$.

\paragraph{Results} The results of our mesh size and resolution runtime tests are shown in Fig. \ref{fig:performance}. 
In both tests and for all $j$-simplices, our implementation of the analytic derivative consistently outperforms the numerical method for calculating the derivative by $10\sim 100\times$ in the range we tested.

\subsection{Shape Optimization}\label{ssec:optim}
We demonstrate the utility of the DDSL through the task of shape optimization. 
Since many physical characteristics depend on shape, shape optimization is an important and challenging task across many fields of science and engineering. 
We show that the DDSL allows us to accomplish this shape optimization task due to the analytic nature of its derivative.

\paragraph{General Experiment Setup} 
We pre-process each shape into a polygon of the shape's boundary. 
The polygons are rasterized using the DDSL. 
We train neural networks on the raster images, and we use the gradients out of these neural networks for the shape optimization task.

Using gradient descent, we optimize a shape to a prescribed target value, which can be a shape classification or a physical quantity. 
Since we implemented the DDSL as a differentiable neural network layer, we can obtain the gradient of the target value with respect to the original shape directly from the neural network.
Rather than directly manipulating vertices, we further propagate this gradient to control points attached to the original shape for enhanced robustness. 
Each control point has 3 degrees of freedom: translation in the $x$ and $y$ directions, and rotation about the point.
More details about the control points are given in Sec. \ref{assec:controlpoints}.
We iterate the shape optimization process until the loss converges to zero. 

\begin{figure}[t]
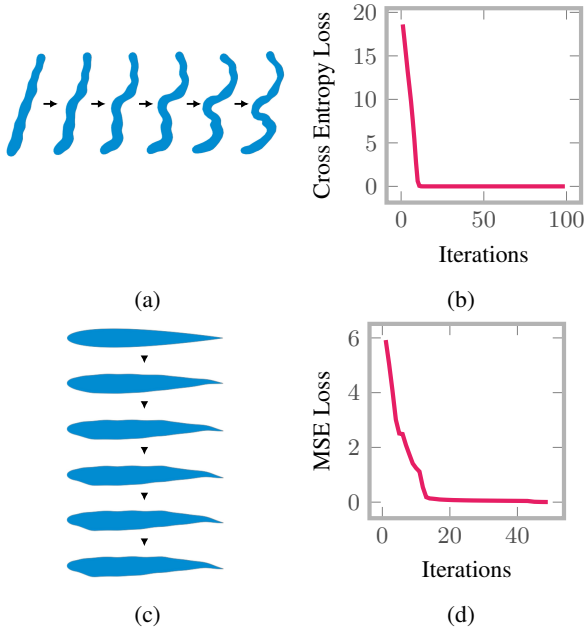
 
    \begin{subfigure}[b]{0.5\linewidth}
    \centering
    \raisebox{4em}{\input{optim1to3/optim1to3_fig.tex} }  
    \caption{}
    \label{MNIST_shapes}
    \end{subfigure}%
    \begin{subfigure}[b]{0.5\linewidth}
    \begin{tikzpicture}[font=\small]
    \begin{groupplot}[
    width = \textwidth,
    height = \textwidth,
    scale=1,
    group style = {
        group size = 1 by 1,
        xlabels at = edge bottom,
        ylabels at = edge left,
        xticklabels at = edge bottom,
        yticklabels at = edge left},
    xlabel = Iterations,
    ylabel = Cross Entropy Loss,
    every outer x axis line/.append style = {black!30, ultra thick},
    every outer y axis line/.append style = {black!30, ultra thick},
    every tick label/.append style = {black!70}]
    \nextgroupplot[]
    \addplotLoss{optim1to3}{iteration}{loss}{main1}
    \end{groupplot}
\end{tikzpicture}
    \caption{}
    \label{fig:MNIST_loss}
    \end{subfigure}
    \begin{subfigure}[b]{0.5\linewidth}
    \centering
    \input{optim_airfoil/optim_airfoil_fig.tex}
    \caption{}
    \label{airfoil_shapes}
    \end{subfigure}%
    \begin{subfigure}[b]{0.5\linewidth}
    \begin{tikzpicture}[font=\small]
    \begin{groupplot}[
    width = \textwidth,
    height = \textwidth,
    scale=1,
    group style = {
        group size = 1 by 1,
        xlabels at = edge bottom,
        ylabels at = edge left,
        xticklabels at = edge bottom,
        yticklabels at = edge left},
    xlabel = Iterations,
    ylabel = MSE Loss,
    every outer x axis line/.append style = {black!30, ultra thick},
    every outer y axis line/.append style = {black!30, ultra thick},
    every tick label/.append style = {black!70}]
    \nextgroupplot[]
    \addplotLoss{optim_airfoil}{iteration}{loss}{main1}
    \end{groupplot}
\end{tikzpicture}
    \caption{}
    \label{fig:airfoil_loss}
    \end{subfigure}
\caption{Optimization of (a), (b): a `1' from the MNIST dataset to a `3' by minimizing the cross-entropy between the input and the target class. (c), (d): the NACA 0012 airfoil (with an original lift-drag ratio of 0) to a lift-drag ratio of 95.9. The airfoil is set at an angle of attack of zero, and the Reynolds number is set to $1 \times 10^6$.}
\label{fig:optim}
\end{figure}

\paragraph{MNIST}
We first demonstrate shape optimization using the DDSL with the MNIST dataset of handwritten digits.
Rather than using the traditional pixel images, we use polygons of the digits as inputs. The polygon form of MNIST digits can be acquired by contouring the original images. The objective of this experiment is to optimize a digit in the MNIST dataset to a target digit.

\paragraph{Airfoils}
We further illustrate the functionality of the DDSL with the more practical task of aerodynamic shape optimization.
For this experiment, we optimize an airfoil to a prescribed lift-drag ratio, which is related to the efficiency of an aerodynamic body.
We use the \url{airfoiltools.com} database of consisting of 1,636 airfoils of aircraft wings and turbine blades, along with precomputed physical quantities such as drag and lift coefficients at different angles of attack and Reynolds numbers, acquired from CFD simulations.
Airfoils are originally represented as polygons and rasterized using the DDSL.
We then train a neural network to predict lift-drag ratios of airfoils at specific angles of attack and Reynolds numbers and use this neural network for the shape optimization task.
When optimizing the airfoil shape, we specify the angle of attack of the airfoil and the Reynolds number of the flow.

\begin{figure*}[ht!] 
\centering
\input{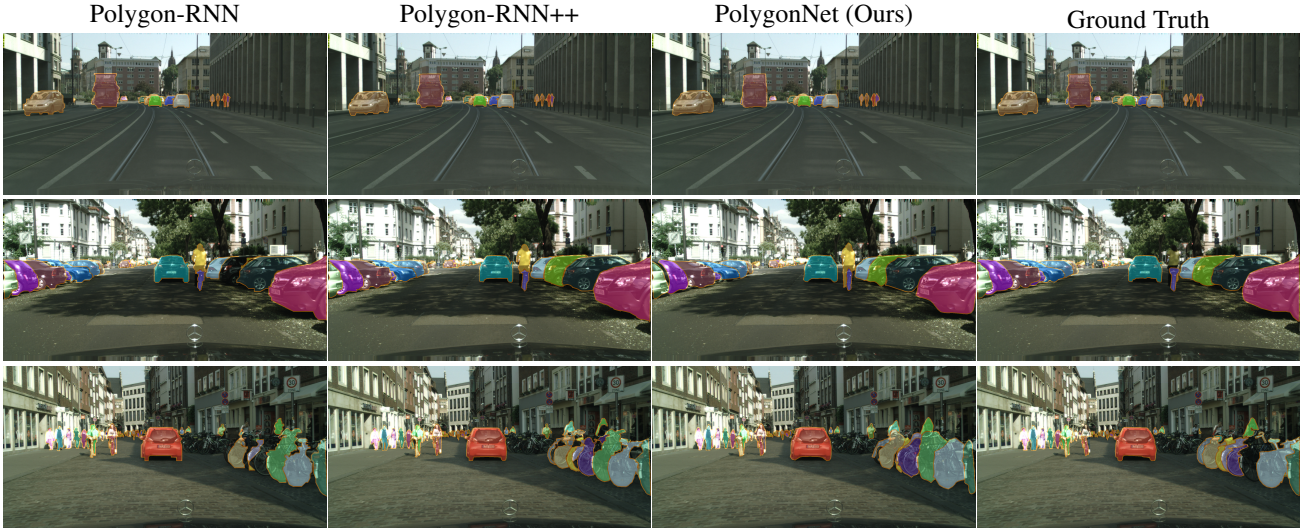}
\caption{Visualization of image segmentation results. Ground-truth bounding boxes are given for all models to create image crops as inputs to the networks.}
\label{fig:visualize}
\end{figure*}

\begin{table*}[ht!]
    \centering
    \begin{tabular}{l|c|c|c|c|c|c|c|c|c}
        \toprule
         Model & Bicycle & Bus & Person & Train & Truck & \scalebox{.7}[1.0]{Motorcycle} & Car & Rider & Mean \\
         \midrule
         SquareBox \scalebox{.8}[1.0]{\cite{castrejon2017annotating}}     & 35.41 & 53.44 & 26.36 & 39.34 & 54.75 & 39.47 & 46.04 & 26.09 & 40.11 \\
         Dilation10 \scalebox{.8}[1.0]{\cite{YuKoltun2016}} & 46.80 & 48.35 & 49.37 & 44.18 & 35.71 & 26.97 & 61.49 & 38.21 & 43.89 \\
         DeepMask  \scalebox{.8}[1.0]{\cite{pinheiro2015learning}}      & 47.19 & 69.82 & 47.93 & 62.20 & 63.15 & 47.47 & 61.64 & 52.20 & 56.45 \\
         SharpMask \scalebox{.8}[1.0]{\cite{pinheiro2016learning}}      & 52.08 & 73.02 & 53.63 & 64.06 & 65.49 & 51.92 & 65.17 & 56.32 & 60.21 \\
         Polygon-RNN \scalebox{.8}[1.0]{\cite{castrejon2017annotating}} & 52.13 & 69.53 & 63.94 & 53.74 & 68.03 & 52.07 & 71.17 & 60.58 & 61.40 \\
         Polygon-RNN++ \scalebox{.8}[1.0]{\cite{acuna2018efficient}}   & \textbf{63.06} & 81.38 & \textbf{72.41} & 64.28 & \textbf{78.90} & 62.01 & \textbf{79.08} & \textbf{69.95} & 71.38 \\
         \midrule
         PolygonNet (Ours)   & 62.26 & \textbf{84.38} & 68.62 & \textbf{82.42} & 76.57 & \textbf{63.57} & 78.08 & 64.10 & \textbf{72.50} \\
         \bottomrule
    \end{tabular}
    \caption{Comparison of Cityscape image segmentation IoU against baseline algorithms on test set. }
    \label{tab:segresults}
\end{table*}

\paragraph{Results}
We show some iterations of the shape optimization process for the MNIST and airfoil experiments as well as graphs showing the loss over each iteration in Figs. \ref{fig:optim}, respectively.
The success of the DDSL in the shape optimization task is most intuitively clear in the MNIST experiment, where the original digit, `1,' is transformed into a `3.' 
In the airfoil experiment, the lift-drag ratio increased, as desired.
The optimized shape is an airfoil with its trailing edge deflected downwards, resembling an aircraft deploying its flaps at takeoff to increase lift.
Both experiments exhibit a monotonic decrease in loss, which converges to zero, confirming that optimization was achieved.

\subsection{Segmentation Mask Generation}\label{ssec:seg}
To further illustrate applications of the DDSL layer in deep learning applications, we experiment on the task of image segmentation by generating polygonal masks. In contrast to conventional segmentation frameworks that output pixel masks, directly predicting polygons allows for a more efficient and flexible output structure, and has been shown to be effective in assisting human annotators in labeling new datasets \cite{castrejon2017annotating, acuna2018efficient}.   

\paragraph{Experiment Setup} For direct comparison with state-of-the-art, we follow the experiment setup of \cite{castrejon2017annotating} and \cite{acuna2018efficient} for predicting polygonal masks. In contrast to the conventional setup of instance segmentation, we assume crops of input images given ground-truth bounding boxes, and we output the corresponding polygonal masks using our neural network. Following the two studies, we train and test our model on the Cityscapes dataset \cite{cordts2016cityscapes}. The Cityscapes dataset is one of the most comprehensive benchmarks for instance segmentation, containing 2975 training, 500 validation, and 1525 test images labeled with 8 semantic classes.  We follow the two studies for an alternative split of the original dataset, since the original test images do not provide ground-truth instances. The new partitions consists of 40174 / 3448 / 8440 image crops of train/validation/test sets, each of size $224\times 224$. 

\paragraph{Training} We use two losses for training the model, a multi-resolution rasterization loss, and a smoothness loss. The losses are defined as:
\begin{align}
    \mathcal{L}_{\text{mres}} &= \sum_{i,res} ||D_{res}(G_{\theta}^{(i)}(x)) - D_{res}(y)||_1 \\
    &i\in \{0, 1, 2, 3\}, res \in 
    \{224, 112, 56, 28\} \nonumber\\
    \mathcal{L}_{\text{smooth}} &= \frac{1}{n} \sum_{j}^n (\frac{A_j(G_{\theta}^{(3)}(x))}{\pi} - 1)^2 \\
    \mathcal{L} &= \mathcal{L}_{\text{mres}}+\lambda \mathcal{L}_{\text{smooth}}
\end{align}
where $D_{res}$ is DDSL rasterization at resolution $res$, $G_{\theta}^{(i)}$ is the polygon output from the polygon generator network parameterized by $\theta$, up to level $i$, $x$ and $y$ are the input images and the ground-truth polygons, $A_j$ is the $j$-th angel of the polygon, and $\lambda$ is the smoothness penalty term.
We train the model (see Fig. \ref{fig:dlarch2}) end-to-end using the loss defined above. We weight the loss of each class inversely proportional to the label frequencies in the training set. See more details in Appendix \ref{assec:seg}.
\paragraph{Results} We evaluate our model against state-of-the-art models and detail the results in Table \ref{tab:segresults}, where we evaluate runtime on a single Titan X (Pascal) GPU. We provide a visual comparison in Fig. \ref{fig:visualize}. Our model surpasses state of the art for class-averaged IoU. In particular, the simplicity of our network architecture is highlighted in Table~\ref{tab:segtime}. While Polygon-RNN++ was unable to propagate gradients through IoU scores, it uses IoU as a reward to an additional reinforcement learning model, which adds additional complexities to the overall architecture. It also uses additional graph neural network to upsample and finetune the polygons. Due to the differentiable rasterization loss, our model uses a single CNN-based polygon generator. In comparison to Polygon-RNN++, our model achieves a ~100x speed-up with a quarter of the total model parameters. 

\begin{table}[ht!]
    \centering
    \begin{tabular}{l l l}
        \toprule
         Model & \# Params & Runtime (s) \\
         \midrule
         Polygon-RNN & 58M & $2.0332\pm 0.0168$ \\
         Polygon-RNN++ & 100M & $2.3241\pm 0.0181$ \\
         PolygonNet (Ours) & \textbf{24M} & $\textbf{0.0287}\pm 0.0022$ \\
         \bottomrule
    \end{tabular}
    \caption{Comparison of network parameters and evaluation time for a batch of 16 image crops.}
    \label{tab:segtime}
\end{table}
\section{Conclusion}\label{sec:conc}
We propose the DDSL as a differentiable simplex layer for neural networks. We present a unifying framework for differentiable rasterization of arbitrary geometrical signals represented on a simplicial complex. We further show two geometric applications of this method: we can effectively propagate gradients across the DDSL for shape optimization, and we can utilize the DDSL to construct a differentiable rasterization loss that allows for a simple, yet effective, polygon generating network that surpasses state of the art in segmentation IoU as well as runtime and parameter efficiency.

\section{Acknowledgements}
We would like to thank Thomas Funkhouser and Avneesh Sud for helpful discussions. We appreciate help from Ling Huan for providing code and data for benchmarking our results against PolygonRNN++. 
This work is supported by a TUM-IAS Rudolf M\"{o}{\ss}bauer Fellowship and the ERC Starting Grant Scan2CAD (804724).
\clearpage

{\small
\bibliographystyle{ieee_fullname.bst}
\bibliography{egbib}
}

\clearpage

\renewcommand\thesection{\Alph{section}}
\setcounter{page}{1}
\renewcommand\thesubsection{\thesection \arabic{subsection}}
\noindent{\Large \textbf{Appendix}}
\setcounter{section}{0}
\label{sec:appendix}

\vspace{.5em}
In the appendix we provide additional details for deriving the derivative of the NUFT process as well as control point methods (Sec. \ref{asec:math}), network architecture and training details (Sec. \ref{asec:arch}). In Sec. \ref{asec:efficiency} we provide additional computational performance benchmarks for the DDSL layer. In Sec. \ref{asec:3dapp} we showcase additional applications of the DDSL towards 3D applications besided the 2D examples in the main paper. In Sec. \ref{asec:3dvis} we provide additional visualizations for the DDSL rasterization of 3D meshes.

\section{Mathematical Derivations}\label{asec:math}
\subsection{NUFT Derivative Derivation}\label{assec:derivative}
\begin{proof}[Proof of Lemma \ref{lem:dgammadxp}]
    Using Jacobi's formula and chain rule,
    \begin{align}
        \frac{\partial \gamma_n^j}{\partial\bm{x}_p}&=\frac{(-1)^{j+1}}{2\sqrt{2^j(-1)^{j+1}det(\hat{B}_n^j)}}tr\left(adj(\hat{B}_n^j)\frac{\partial\hat{B}_n^j}{\partial \bm{x}_p}\right)\\
        &=\frac{(-1)^{j+1}/2^j}{2\gamma_n^j}\sum_{m=1}^{j+2}\sum_{n=1}^{j+2}\tilde{A}_{mn}\tilde{{D}}_{nm}
        \label{eq:jacobi-expanded}
    \end{align}
    where $\tilde{A}$ is $adj(\hat{B}_n^j)$ and $\tilde{{D}}$ is $\frac{\partial\hat{B}_n^j}{\partial \bm{x}_p}$. Since $\hat{B}_n^j$ is symmetric, its adjunctive and derivative with respect to $\bm{x}_p$ are also symmetric. The elements on the diagonal and the first row and column of $\tilde{D}$ are zero, since the elements in the same positions in $\hat{B}_n^j$ are constant. The elements not in the $(p+1)$th row or the $(p+1)$th column of $\tilde{D}$ are also zero, since the elements in these positions in $\hat{B}_n^j$ do not depend on $\bm{x}_p$. Thus,
    \begin{equation}
        \tilde{D}=\begin{bmatrix}
        0 & \dots & 0 & 0 & 0 & \dots\\
        \vdots & \ddots & \vdots  & \vdots & \vdots \\
        0 & \dots & 0 & \tilde{D}_{p,p+1} & 0 & \dots\\
        0 & \dots & \tilde{D}_{p+1,p} & 0 & \tilde{D}_{p+1,p+2} & \dots \\
        0 & \dots & 0 & \tilde{D}_{p+2,p+2} & 0 & \dots \\
        \vdots & & \vdots & \vdots & \vdots & \ddots 
        \end{bmatrix}
    \end{equation}
    
    Each nonzero element of $\tilde{D}$ is computed as follows:
    \begin{align}
        \tilde{D}_{p+1,n}&=\frac{\partial d_{p,n-1}^2}{\partial \bm{x}_p}=2(\bm{x}_p-\bm{x}_{n-1})\\
        \tilde{D}_{m,p+1}&=\frac{\partial d_{m-1,p}^2}{\partial \bm{x}_p}=2(\bm{x}_p-\bm{x}_{m-1})
    \end{align}

    It follows that the double summation term in Eqn. \ref{eq:jacobi-expanded} simplifies to
    \begin{equation}
        \sum_{m=1}^{j+2}\sum_{n=1}^{j+2}\tilde{A}_{mn}\tilde{{D}}_{nm}=2\sum_{\substack{m=2\\m\neq{p+1}}}^{j+2}\tilde{A}_{p+1,m}\tilde{D}_{p+1,m}
        \label{eq:AD}
    \end{equation}
    
    For clarity and ease of implementation, we modify the indexing in Eqn. \ref{eq:AD} and the derivative of the content distortion factor is finally
    \begin{equation}
        \frac{\partial \gamma_n^j}{\partial\bm{x}_p}=\frac{(-1)^{j+1}/2^j}{\gamma_n^j}\sum_{\substack{m=1\\m\neq{p}}}^{j+1}A_{pm}\bm{D}_{pm}
    \end{equation}
\vspace{-1em}
\end{proof}

\begin{proof}[Proof of Lemma \ref{lem:dSdxp}]
    By the sum rule,
    \begin{equation}
        \frac{\partial S}{\partial\bm{x}_p}=\sum_{t=1}^{j+1}\frac{\partial S_t}{\partial\bm{x}_p}
    \end{equation}
    We examine two cases, when $t=p$ and when $t\neq{p}$. For $t=p$,
    \begin{align}
        \frac{\partial S_t}{\partial \bm{x}_p}=&\frac{1}{\left(\prod_{l=1,l\neq{p}}^{j+1}(\sigma_p-\sigma_l)\right)^2}\bm{k}\nonumber\\
        &\left[\left(\prod_{l=1,l\neq{p}}^{j+1}(\sigma_p-\sigma_l)\right)\left(-ie^{-i\sigma_p}\right)\right.\nonumber\\
        &\left.+e^{-i\sigma_p}\left(\frac{\partial}{\partial \bm{x}_p}\left(\prod_{l=1,l\neq{p}}^{j+1}(\sigma_p-\sigma_l)\right)\right)\right]\\
        =&-\frac{e^{-i\sigma_p}}{\prod_{l=1,l\neq{p}}^{j+1}(\sigma_p-\sigma_l)}\left(i+\sum_{q=1,q\neq{p}}^{j+1}\frac{1}{\sigma_p-\sigma_q}\right)\bm{k}
    \end{align}
    
    For $t\neq{p}$,
    \begin{align}
        \frac{\partial S_t}{\partial \bm{x}_p}=&\frac{\partial}{\partial \bm{x}_p}\left(\frac{e^{-i\sigma_t}}{(\sigma_t-\sigma_1)...(\sigma_t-\sigma_p)...(\sigma_t-\sigma_{j+1})}\right)\\
        =&\left(\frac{e^{-i\sigma_t}}{\splitfrac{(\sigma_t-\sigma_1)...(\sigma_t-\sigma_{p-1})(\sigma_t-\sigma_{p+1})}{...(\sigma_t-\sigma_{j+1})}}\right)\nonumber\\
        &\left(\frac{\partial}{\partial \bm{x}_p}\left(\frac{1}{\sigma_t-\sigma_p}\right)\right)\\
        =&\left(\frac{e^{-i\sigma_t}}{\splitfrac{(\sigma_t-\sigma_1)...(\sigma_t-\sigma_{p-1})(\sigma_t-\sigma_{p+1})}{...(\sigma_t-\sigma_{j+1})}}\right)\nonumber\\
        &\left(\frac{1}{(\sigma_t-\sigma_p)^2}\bm{k}\right)\\
        =&\frac{e^{-i\sigma_t}}{\prod_{l=1,l\neq{t}}^{j+1}(\sigma_t-\sigma_l)}\left(\frac{1}{\sigma_t-\sigma_p}\right)\bm{k}
    \end{align}
    
    Thus,
    \begin{align}
        \frac{\partial S}{\partial\bm{x}_p}=&\sum_{t=1}^{j+1}\frac{\partial S_t}{\partial\bm{x}_p}\\
        =&\left(\sum_{t=1,t\neq{p}}^{j+1}\left(\frac{e^{-i\sigma_t}}{\prod_{l=1,l\neq{t}}^{j+1}(\sigma_t-\sigma_l)}\left(\frac{1}{\sigma_t-\sigma_p}\right)\right)\right. \nonumber \\
        &\left.-\frac{e^{-i\sigma_p}}{\prod_{l=1,l\neq{p}}^{j+1}(\sigma_p-\sigma_l)}\left(i+\sum_{q=1,q\neq{p}}^{j+1}\frac{1}{\sigma_p-\sigma_q}\right)\right)\bm{k}\\
        =&\left(-i\frac{e^{-i\sigma_p}}{\prod_{l=1,l\neq p}^{j+1}(\sigma_p - \sigma_l)}+\sum_{t=1,t\neq p}^{j+1}\frac{1}{\sigma_t - \sigma_p}\right.\nonumber\\
        &\left.\left[\frac{e^{-i\sigma_t}}{\prod_{l=1,l\neq t}^{j+1}(\sigma_t - \sigma_l)}+\frac{e^{-i\sigma_p}}{\prod_{l=1,l\neq p}^{j+1}(\sigma_p - \sigma_l)}\right]\right)\bm{k}\\
        =&\left(-i S_p+\sum_{t=1,t\neq p}^{j+1}\frac{S_t+S_p}{\sigma_t - \sigma_p}\right)\bm{k}
    \end{align}
\end{proof}

\begin{proof}[Derivation of Eqn. \ref{eq:dFdxp}]
    Using the product rule,
    \begin{equation}
        \frac{\partial F_n^j(\bm{k})}{\partial\bm{x}_p}=\rho_ni^j\left(\frac{\partial \gamma_n^j}{\partial\bm{x}_p}S+\frac{\partial S}{\partial\bm{x}_p}\gamma_n^j\right)
        \label{eq:product}
    \end{equation}
    We obtain Eqn. \ref{eq:dFdxp} by substituting Eqns. \ref{eq:dgammadxp} and \ref{eq:dSdxp} into Eqn. \ref{eq:product}.
\end{proof} 

\subsection{Control Points}\label{assec:controlpoints}
We use linear blend skinning to control mesh deformation using control points. The new position of a point $\bm{v}'$ on the shape is computed as the weighted sum of handle transformations applied to its rest position $\bm{v}$:
\begin{align*}
    \bm{v}'=\sum_{j=1}^{m}w_j(\bm{v})\bm{T}_j\begin{pmatrix}\bm{v} \\ 1\end{pmatrix}
\end{align*}
Where $\bm{T}_j$ is the transformation matrix for the $j$-th control point, $w_j(\bm{v})$ is the normalized weight on vertex $\bm{v}$ corresponding to control point $j$. The transformation is represented in homogeneous coordinates, hence the extra dimension.

Consider control points with 3 degrees of freedom: $(t_x, t_y, \theta)$ where $t_x$ and $t_y$ represent translations in $x$ and $y$ and $\theta$ represents rotation around that control point. Hence we have
\begin{align*}
    \begin{cases}
    v_x' =& \sum_{j=1}^{N}w_j(\bm{v})\big(\cos(\theta_j - \tilde{\theta_j})v_x - \sin(\theta_j - \tilde{\theta_j}) v_y \\
    &- \cos(\theta_j - \tilde{\theta_j})c_x+\sin(\theta_j - \tilde{\theta_j})c_y\\
    &+c_x+v_x+t_x\big)\\
    v_y' =& \sum_{j=1}^{N}w_j(\bm{v})\big(\sin(\theta_j - \tilde{\theta_j})v_x + \cos(\theta_j - \tilde{\theta_j}) v_y \\
    &- \sin(\theta_j - \tilde{\theta_j})c_x-\cos(\theta_j - \tilde{\theta_j})c_y\\
    &+c_y+v_y+t_y\big)
    \end{cases}
\end{align*}
Where $\tilde{\theta_j}$ is the original orientation of the control points. It does not matter since we will be taking the derivatives with respect to $\theta$, and $\tilde{\theta_j}$ terms will disappear. The jacobian of $\bm{v}$ with respect to the three degrees of freedom is:
\begin{align*}
    \bm{J} &= \left[\frac{\partial \bm{v}}{\partial t_x}, \frac{\partial \bm{v}}{\partial t_y}, \frac{\partial \bm{v}}{\partial \theta}\right]\nonumber\\
    &= 
    \begin{bmatrix}
    w_j(\bm{v}) & 0 & w_j(\bm{v})(-v_y+c_y) \\
    0 & w_j(\bm{v}) & w_j(\bm{v})(v_x - c_x) \\
    \end{bmatrix}
\end{align*}

\section{Network Architecture and Training Details}\label{asec:arch}
In this section, we detail all the network architectures and training routines for the reader's reference.

\begin{table}[b]
    \centering
    \begin{tabular}{c p{14em}}
        \toprule
        Notation &  Meaning \\
        \midrule
        Conv(a, b, c, d) & Convolutional layer with $a$ input channels, $b$ output channels, kernel size $c$, and stride $d$. \\
        MaxPool(a) & Maximum Pooling with a kernel size of $a$. \\
        ReLU & Rectified Linear Unit activation function. \\
        FC(a, b) & Fully connected layer with $a$ input channels and $b$ output channels. \\
        ResNet-50(a) & ResNet-50 architecture with $a$ output channels. \\
        BN & Batch Normalization. \\
        \bottomrule
    \end{tabular}
    \caption{Network architecture notation list.}
    \label{tab:nnnotation}
\end{table}

\subsection{MNIST}\label{assec:mnist}
We use a standard LeNet-5 architecture with 3 convolutional layers and 2 fully connected layers.
\paragraph{Network Architecture}
The input is a 28x28 pixel image, which is normalized according to the mean and standard deviation of the entire dataset.
The network architecture is as follows:

Conv(1, 10, 5, 1) + MaxPool(2) + ReLU $\rightarrow$
Conv(10, 20, 5, 1) + Dropout + MaxPool(2) + ReLU $\rightarrow$
FC(320, 250) + ReLU $\rightarrow$
Dropout $\rightarrow$
FC(250, 10)

Total number of parameters: 88,040

\paragraph{Training Details}
We train the neural network with a batch size of 64 and an initial learning rate of $1\times10^{-2}$ with a decay of $0.5$ per 10 epochs.
We use the Stochastic Gradient Descent optimizer with a momentum of $0.5$ and a cross entropy loss.

\subsection{Airfoil}\label{assec:airfoil}
We use ResNet-50 \cite{he2016deep} followed by three fully connected layers to predict the lift-drag ratio on the airfoil.
\paragraph{Network Architecture}
The input is a 224x224 pixel image of the airfoil.
For each piece of data, we append the Reynolds number and angle of attack after ResNet-50 and before the fully connected layers.
The network architecture is as follows:

ResNet-50(1000) + BN + ReLU $\rightarrow$ 
append Reynolds number and angle of attack $\rightarrow$ 
FC(1002, 512) + BN + ReLU $\rightarrow$ 
FC(512, 64) + BN + ReLU $\rightarrow$ 
FC(64, 32) + BN 

Total number of parameters: 26,100,345

\paragraph{Training Details}
We train the neural network with a batch size of 240 and an initial learning rate of $1\times10^{-2}$ with a decay of $1\times10^{-1}$ per 20 epochs.
We use the Adam optimizer and a mean squared error loss.

\subsection{Polygon Image Segmentation}\label{assec:seg}
We present a novel polygon decoder architecture that is paired with a standard pre-trained ResNet50 as input. 

\paragraph{Network Architecture}
The model architecture is detailed in Fig. \ref{fig:dlarch2}. All ground-truth polygons are normalized to the range [0,1) corresponding to the relative positions within the bounding boxes. Using this network architecture, we first predict the three $(x,y)$ coordinates associated with the base triangle. Then, we progressively predict the offsets of the vertices in the next polygon hierarchy (See Fig. \ref{fig:polygen}). The resulting polygon is rasterized with the DDSL to compute the rasterization loss compared with the rasterized target. Smoothness loss can be directly computed based on the vertex positions and does not require rasterization. 

Total number of parameters: 24,274,426

\paragraph{Training Details}
We train the network end-to-end, with a batch size of 48, learning rate of $10^{-3}$ for 200 epochs. We use a smoothness penalty of $\lambda=1$. We use the Adam optimizer.

\section{Additional Computational Efficiency Tests}\label{asec:efficiency}
In addition to the computational speed benchmarks in Fig. \ref{fig:performance} highlighting the performance gain of analytic derivative computation over numerical derivatives, we perform additional tests for 2D and 3D computation speeds on more complex polygons and meshes to show the applicability of DDSL to 2D and 3D computer vision problems.

\begin{table}[h!]
    \centering
    \begin{tabular}{c|c|c|c|c|c}
        \toprule
        Res$^2$ & 16 & 32 & 64 & 128 & 256 \\
        \midrule
        Fwd Time (ms) & 2.30 & 1.88 & 2.48 & 5.02 & 20.13 \\
        \midrule
        Bwd Time (ms) & 4.33 & 3.80 & 5.93 & 16.69 & 59.15 \\
        \bottomrule
    \end{tabular}
    \vspace{-0.2cm}
    \caption{2D Computational speed (polygon w/ 250 edges).}
\end{table}
\vspace{-2em}
\begin{table}[h!]
    \vspace{1em}
    \begin{tabular}{c|c|c|c|c}
        \toprule
        Res$^3$ & 4 & 8 & 16 & 32 \\
        \midrule
        Fwd Time (ms) & 9.88 & 9.32 & 14.21 & 78.62 \\
        \midrule
        Bwd Time (ms) & 14.47 & 10.06 & 34.26 & 239.51 \\
        \bottomrule
    \end{tabular}
    \vspace{-0.2cm}
    \caption{3D Computational speed (tri-mesh w/ 1300 faces).}
    \label{tab:23d}
\end{table}

\section{3D Geometric Applications}\label{asec:3dapp}
To showcase the generalizabilty of the DDSL to 3D domain, we demonstrate its application in two separate 3D tasks that utilze the differentiablity of the simplex rasterization layer.

\subsection{3D Rotational Pose Estimation}
In Fig.~\ref{fig:pose}, we use DDSL to create a differentiable volumetric loss comparing current and target shapes, the gradients of which can be backpropagated to the pose. More specifically, we parameterize the rotational pose as a quaternion $\bm{q} = a+b\hat{\bm{i}}+c\hat{\bm{j}}+d\hat{\bm{k}}, \quad s.t. ||\bm{q}||_2 = 1$. The rasterization loss is defined as:
\begin{align*}
    \mathcal{L}(\bm{q}) = ||D_{32}(V(\bm{q})) - D_{32}(V_{tg})||_1
\end{align*}
where $D_{32}$ is the rasterization operator at resolution $32^3$ and $V_{tg}$ is the target mesh.

\begin{figure}[h!]
    \centering
    \includegraphics[width=.8\linewidth]{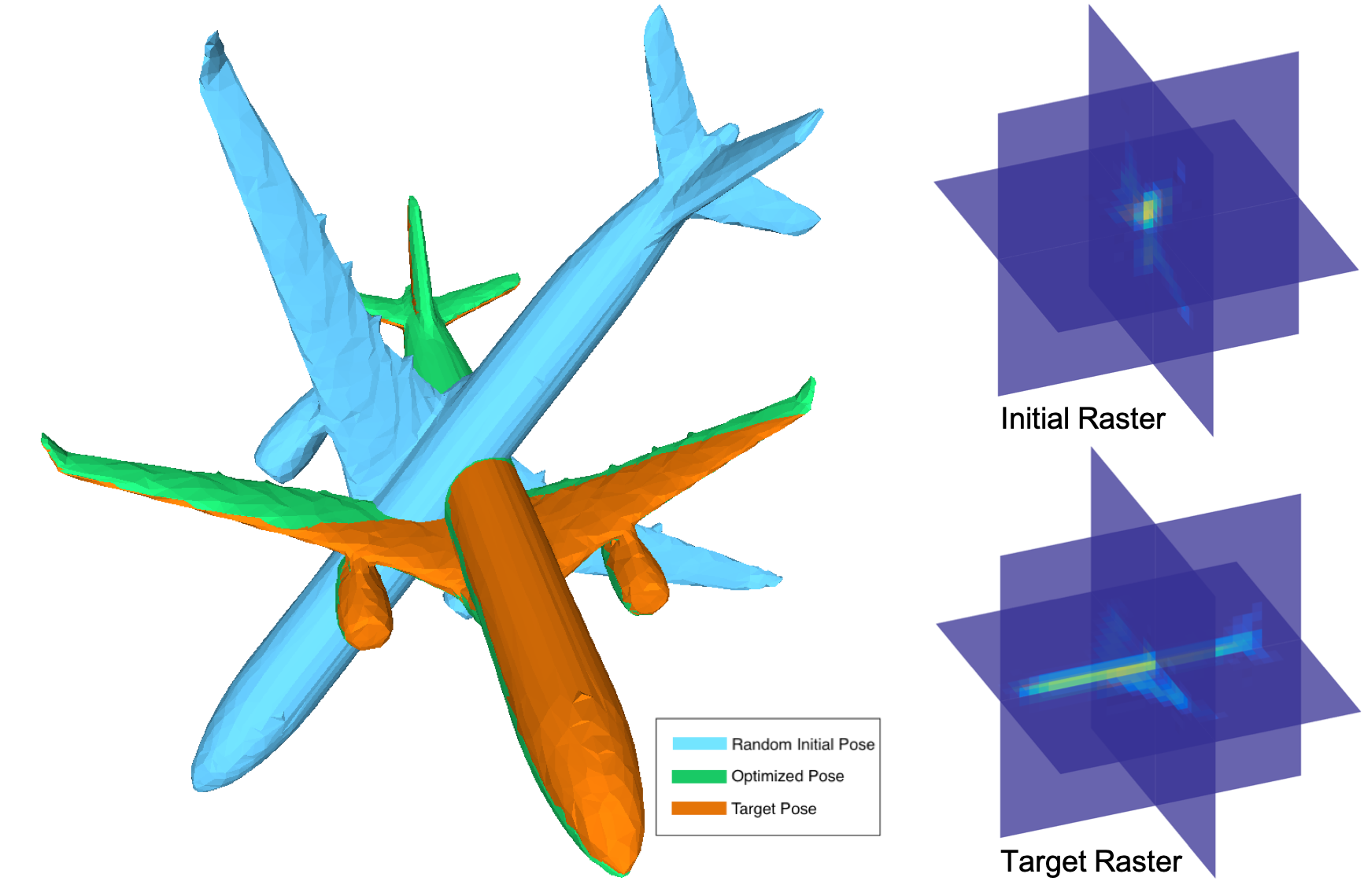}
    \caption{Mesh pose and rasters before and after opt.}
    \label{fig:pose}
\end{figure}

Although the volumetric rasterization loss is not a globally convex loss for pose alignment, with certain initialization of the target poss, the pose can be estimated by minimizing the DDSL rasterization loss.

\subsection{Single Image Mesh Estimation}
In Fig.~\ref{fig:vis}, we evaluate our method in the context of 3D deep learning. 
Our model consists of an image encoder from ResNet18, spherical convolutions \cite{jiang2018spherical} for generating a distortion map for a spherical mesh, and a loss function which is a weighted sum of DDSL rasterization loss (at $32^3$ resolution), Chamfer loss from point samples, Laplacian regularization loss, and Edge length regularization loss. We train on the airplane category in ShapeNet dataset, 
with (w/) and without (w/o) DDSL loss. 
We evaluate using accuracy, completeness, and chamfer distance metrics (see Tab. \ref{tab:eval}).

Since surface based Chamfer distance does not signal the network to produce consistently oriented surfaces and does not consistently enclose volume, it leads to incorrectly oriented surfaces. DDSL loss effective regularizes surface orientation based on the volume enclosed according to the surface orientations, and improves overall results.

\begin{table}[h]
    \begin{tabular}{l c c c}
        \toprule
        DDSL & Accuracy & Complete & Chamfer \\
        \midrule
        w/o  & 8.47 & 9.84 & 9.16 \\
        w/ & \textbf{2.15} & \textbf{1.83} & \textbf{1.99} \\
        \bottomrule
    \end{tabular}
    \caption{Evaluation resultsn($\times 10^-2$).}
    \label{tab:eval}
\end{table}

\begin{figure}[h]
    \centering
    \begin{subfigure}{.5\linewidth}
    \centering
    \includegraphics[width=\linewidth]{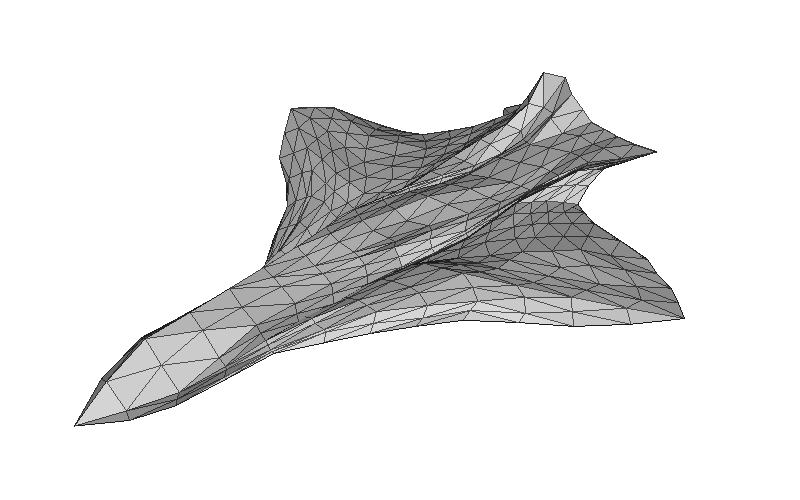}
    \caption{w/ DDSL}
    \end{subfigure}%
    \begin{subfigure}{.5\linewidth}
    \centering
    \includegraphics[width=\linewidth]{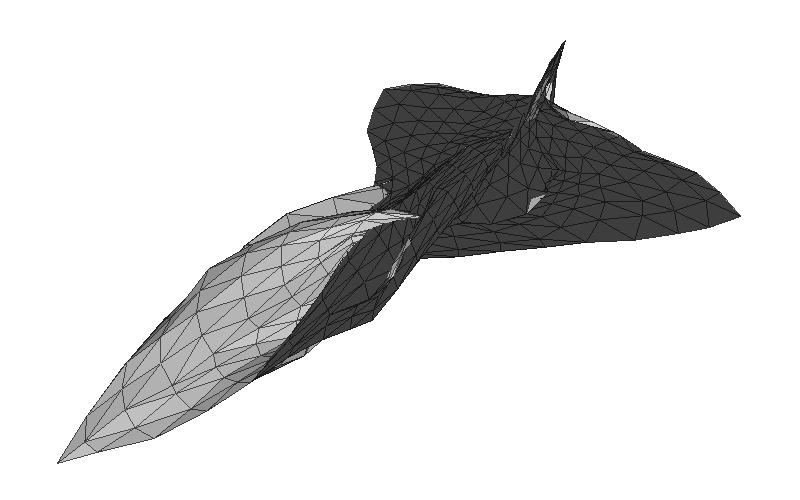}
    \caption{w/o DDSL}
    \end{subfigure}
    \caption{Qualitative visualization of generated samples.}
    \label{fig:vis}
\end{figure}

\section{Additional 3D Visualizations}\label{asec:3dvis}
We provide visualizations for rasterizing 3D shapes, rasterizing the enclosed volume as well as the surface mesh.
\begin{figure}[t]
\centering
\input{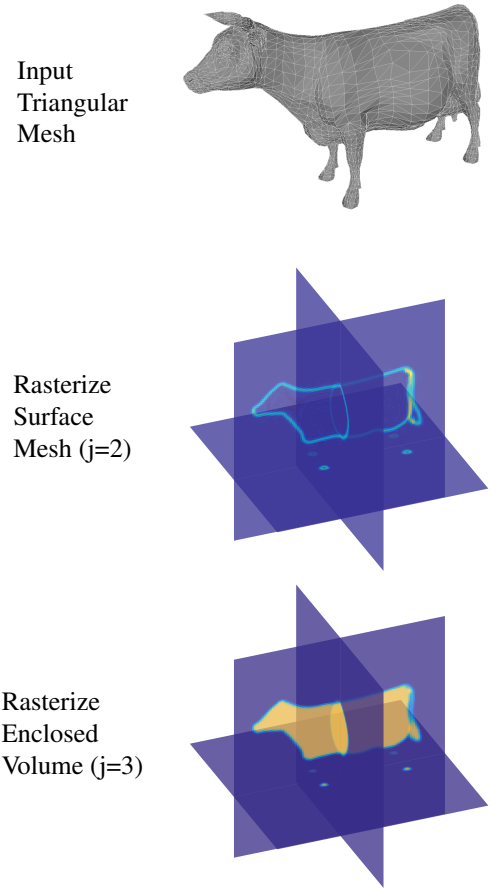}
\caption{In this example above, the input is a watertight triangluar mesh represented by vertices and faces. It can be rasterized in-situ in a 3-dimensional grid differentiably. The value is approximately 0 or 1 indicating signal densities.}
\end{figure}

\end{document}